\documentclass{article}
\usepackage[colorlinks]{hyperref} 
\usepackage{hyperref}
\usepackage[initials]{amsrefs}
\usepackage{amsmath}
\usepackage{amsthm}
\usepackage[utf8]{inputenc}
\usepackage{url}
\usepackage{parskip}
\usepackage{booktabs}
\usepackage{lscape}
\usepackage{graphicx}
\usepackage{multirow}
\usepackage{longtable}
\usepackage{amsfonts}
\usepackage{bm} 
\usepackage{authblk}
\usepackage{derivative}
\usepackage{amssymb}
\usepackage{changepage}
\newtheorem{theorem}{Theorem}[section]
\newtheorem{proposition}[theorem]{Proposition}

\newtheorem{definition}[theorem]{Definition}
\newcommand{\inprod}[2]{\langle \bm{#1}, \bm{#2} \rangle}
\newcommand{\norm}[1]{\| \bm{#1} \|}

\newcommand{\cossim}[2]{\frac{\langle \bm{#1}, \bm{#2} \rangle}{\| \bm{#1} \| \| \bm{#2} \|}}

\newcommand{\cosdist}[2]{1 - \frac{\langle \bm{#1}, \bm{#2} \rangle}{\| \bm{#1} \| \| \bm{#2} \|}}

\newcommand{\metric}[3]{\frac{\gamma}{\| \bm{#1} \|}\left( \delta_{#2#3} - \frac{#1_{#2}#1_{#3}}{\| \bm{#1} \|^2} \right)}

\newcommand{\DOneCos}[3]{e^{-\frac{1}{\tau_{#1}}\left(\cosdist{#2}{#3}\right)}}

\newcommand{\DOneArccos}[3]{e^{-\frac{2 \arccos{\left \langle\sqrt{\bm{#2}} ,\sqrt{\bm{#3}} \right \rangle}}{\tau_{#1}}}}

\newcommand{\DSumCos}[3]{\sum_{#1 \in SV }\DOneCos{#1}{#2}{#3}}

\newcommand{\DSumArccos}[3]{\sum_{#1 \in SV }\DOneArccos{#1}{#2}{#3}}

\newcommand{\dercossim}[4]{\sum_{#4 \in SV} \frac{1}{\tau_{#4}} \DOneCos{#4}{#1}{#2} \frac{#1_{#4}\norm{\bm{#1}}^2 - #1_{#3}\inprod{#1}{#1_{#4}}}{\norm{#1}^3\norm{#1_{#4}}}}

\newcommand{\derarccos}[4]{\sum_{#4 \in SV} \frac{1}{\tau_{#4}}\DOneArccos{#4}{#1}{#2}\frac{1}{\sqrt{1- \inprod{\sqrt{#1}}{\sqrt{#2}}^2}}\frac{\sqrt{#1_{#4#3}}}{\sqrt{#1_{#3}}}}

\begin{document}

\title{Conformal Transformation of Kernels: A Geometric Perspective on Text Classification}

\author[1]{IOANA RĂDULESCU (LĂZĂRESCU)}
\author[2]{ALEXANDRA BĂICOIANU}
\author[1,3]{ADELA MIHAI \thanks{Corresponding author: adela.mihai@unitbv.ro}}
\affil[1]{Interdisciplinary Doctoral School, Transilvania University of Bra\c{s}ov, Romania, \texttt{ioana.radulescu@unitbv.ro}}
\affil[2]
{Department of Mathematics and Computer Science, 
Faculty of Mathematics and Computer Science, Transilvania University of Bra\c{s}ov, Romania,
\texttt{a.baicoianu@unitbv.ro}}
\affil[3]{Department of Mathematics and Computer Science,
Technical University of Civil Engineering Bucharest, Romania
\texttt{adela.mihai@unitbv.ro; adela.mihai@utcb.ro}
}

\date{}
\maketitle
\begin{abstract}
In this article we investigate the effects of conformal transformations on kernel functions used in Support Vector Machines. Our focus lies in the task of text document categorization, which involves assigning each document to a particular category. We introduce a new Gaussian Cosine kernel alongside two conformal transformations. Building upon previous studies that demonstrated the efficacy of conformal transformations in increasing class separability on synthetic and low-dimensional datasets, we extend this analysis to the high-dimensional domain of text data. Our experiments, conducted on the Reuters dataset on two types of binary classification tasks, compare the performance of Linear, Gaussian, and Gaussian Cosine kernels against their conformally transformed counterparts. The findings indicate that conformal transformations can significantly improve kernel performance, particularly for sub-optimal kernels. Specifically, improvements were observed in 60\% of the tested scenarios for the Linear kernel, 84\% for the Gaussian kernel, and 80\% for the Gaussian Cosine kernel. In light of these findings, it becomes clear that conformal transformations play a pivotal role in enhancing kernel performance, offering substantial benefits.
\end{abstract}

\textbf{Keywords and Phrases:} Text classification, Support Vector Machines, Riemannian geometry, kernels, conformal transformation.

\textbf{2020 AMS Mathematics Subject Classification:} 62B11, 53B12, 62H30

\section{Introduction}

The volume of information available worldwide in electronic format grows exponentially, requiring effective organization and retrieval methods. Therefore, text classification, the task of assigning documents to predefined categories, finds applications across multiple domains. These include topic detection, spam email filtering, SMS spam filtering, author identification, web page classification and sentiment analysis \cites{TextClassApp2, TextClassApp4uysal2016, rill2014App5Topic, gunal2006App6SpamEmail, idris2014App7EmailSpam, uysal2013AppSMSSpam, zhang2014AppAuthId}.

A wide variety of techniques have been proposed for text classification, as it can be seen in the extensive survey by  Aggarwal and Zhai \cite{TextClassAlgSurvey}. We will focus on the Support Vector Machine (SVM), a powerful technique for pattern recognition \cites{joachims1998text, VapnikSVM}. SVM is a type of machine learning algorithm that can handle high-dimensional data. It is efficient because its decision boundary depends only on a subset of the training samples, called support vectors. SVM is effective for both linearly and non-linearly separable data. By mapping input data into a higher-dimensional space and utilizing the kernel, SVM can surpass the limitations of linear classifiers and handle more complex or non-linear datasets. The strength of the SVM lies in the fact that we do not need to compute the mapped patterns explicitly, and only need to know the inner products between the mapped patterns. This method is called the "kernel trick" \cites{scholkopf2000kernelTrick, cristianini2002supportTrick}.

The effectiveness of text classification significantly depends on the document representation and the choice of kernel. This article delves into the exploration of these aspects, focusing on a geometric perspective. We consider representing text documents as draws from a multinomial distribution. A kernel for the multinomial distribution was introduced by Lafferty and Lebanon in \cite{diffusion} based on the heat equation on a Riemannian manifold defined by the Fisher information metric associated with a statistical model. We considered this kernel for our study, however we found that the Diffusion kernel does not induce a Riemannian metric. Therefore we propose a kernel that we call the \textit{Gaussian Cosine} (GC) kernel, based on the cosine distance, that exhibits the desired property of inducing a Riemannian metric and is appropriate for the multinomial geometry. Amari and Wu explored in \cites{amariInfGeo, AmariWu1999, WuAmari} the Riemannian geometry induced by a kernel and proposed an efficient method to increase the separability of classes by a conformal transformation of the kernel. 
This article applies these techniques inspired by information-geometry to text classification. An appropriate conformal transformation of the Linear and Gaussian kernel is proposed. Additionally, a specific transformation dedicated to multinomial geometry is analyzed for the GC kernel. 

The remaining sections are organized as follows. In Section \ref{Sec:Preliminaries} we present the SVM classifier, focusing on the Riemannian structure induced by a kernel function. The conformal transformation technique that enlarges the separation between classes is presented. The geometry of the multinomial distribution is presented, together with an appropriate distance that will be used in the proposed conformal transformation. 
In Section \protect{\ref{Sec:ProposedMethod}} we propose the new GC kernel and we give proof that it is a Mercer kernel. We propose two conformal transformations. One is based on the cosine distance and will be applied to the Linear, Gaussian, and GC kernels. The second one is based on the geodesic distance on the sphere and will be applied to the GC kernel. We derive the formula for the induced Riemannian metric and for the one induced by the conformal transformation for the GC kernel. We present the algorithm that will be used in our experiments. In Section \protect{\ref{Sec:Experiments}} we describe the dataset, the pre-processing steps and the setup of our experiments. The results of the proposed method applied on text documents categorization using the Reuters dataset are detailed in Section \ref{Sec:Results}. The accuracy of the original and modified kernels are discussed. We end with conclusions in Section \protect{\ref{SectionConclusions}}.

\section{Preliminaries}\label{Sec:Preliminaries}
\subsection{Support Vector Machines}\label{SectionSVM}

We consider a binary classification problem. We start from a set of training points $\{(x_i, y_i)\}$,
$i = \overline{1...n}$, where $x_i$ is a vector in the input space $\mathcal{X} \subset \mathbb{R}^n$ and $y_i$ are the labels, with $y_i \in \{-1,1\}$, corresponding to one of the two classes. SVM classifies a point into one of the two classes according to the sign of the decision function
\begin{equation}
        f(\bm x) = \bm w \cdot \bm x + b.
\end{equation}

The parameters $\bm w$ and $b$ are found by maximizing the margin to the hyperplane separating the two classes, which can be solved using Lagrange multipliers. We refer the reader to \cites{smola2004tutorial,burges1998tutorial} for details. This leads to a decision function of the form
 \begin{equation}\label{OptOutputFunction}
        f(\bm x) = \sum_{i \in SV} \alpha_iy_i\bm x_i \cdot \bm x + b,
\end{equation}
with $\alpha_i > 0$ real constants. 
In the above equation \eqref{OptOutputFunction} the summation goes over all support support vectors. For a simplified notation, we wrote 
$\sum_{i \in SV}$ which means $\sum_{x_i \in \text{ \{Support Vectors\} }}$.
For linearly separable data, the points which are closest to the separating hyperplane satisfy $f(\bm x) = \pm 1$ and are called the \textit{support vectors}. It is important to notice the function from Equation \eqref{OptOutputFunction} depends only on the support vectors. The constants $\alpha_i$ are non-zero only for the support vectors and zero otherwise.

When data is not linearly separable, the SVM maps the input data into a higher dimensional space $\mathbb{R}^m, m>n$, where $m$ can even be infinite, using a non-linear mapping $\bm \phi$. In this new space, the problem becomes linearly separable in the form of 
\begin{equation}
        f(\bm x) = \bm w \cdot \bm {\phi(x)} + b.
\end{equation}

The embedding functions $\bm\phi$ do not need to be computed. SVM uses a method called the "kernel trick" \cites{scholkopf2000kernelTrick, cristianini2002supportTrick} to avoid the difficulty of finding good embedding functions. Considering the inner product after the embedding,
 \begin{equation}\label{KernelInner}
        K(\bm x, \bm x') = \langle \bm {\phi(x), \phi(x')\rangle},
\end{equation}
the optimal output function for the SVM becomes
\begin{equation}\label{OutputFunctionKernel}
    f(\bm x) = \sum \alpha_iy_iK(\bm x_i, \bm x) + b.
\end{equation}
When the problem is considered too costly to separate, a bound can be imposed on the multipliers $\alpha_i$ by the condition $\alpha_i \leq C$, for some constant $C > 0$, called the \textit{soft-margin} parameter.  
In order to express a kernel as an inner product in the form of Equation \eqref{KernelInner}, it needs to be a Mercer kernel, for which we give the following characterization theorem (Mercer's Theorem \cites{mercer1909, vapnik1999}):

\begin{theorem}\label{Mercer}
    Let $\mathcal X$ be a subset of $\mathbb R^n$, $\mathcal{X}\subset \mathbb{R}^n$.
    The function $k : \mathcal X^2 \to \mathbb R$ is a \textit{Mercer kernel function} (also known as kernel function) if:
    \begin{enumerate}
        \item it is symmetric: $k(\bm x, \bm y) = k(\bm y, \bm x)$, 
        \item and its corresponding kernel matrix for any finite sequence \\
        $S = \{\bm x_1, \bm x_2, \dots, \bm x_m\}, \bm x_i \in \mathcal X, i=\overline{1,m}$ , given by
        \begin{equation}\label{KernelMatrix}
         \bm K(i,j) = k(\bm x_i, \bm x_j ), \qquad \forall i,j \in \{1,\dots,m\}  
        \end{equation}
        is positive semi-definite: $\bm K\succcurlyeq0$.
    \end{enumerate}
The corresponding kernel matrix of a Mercer kernel is called a Mercer kernel matrix.
\end{theorem}

In order to check that we have a valid Mercer kernel, we need to prove it satisfies the properties from Theorem \ref{Mercer}. 

The following proposition gives some important kernel properties \cite{ShaweTaylorCristianini2004}.
\begin{proposition}\label{propertiesK}
Let $\mathcal X$ be a subset of $\mathbb R^n$, $\mathcal{X}\subset \mathbb{R}^n$.
Let $k_1$ and $k_2$ be Mercer kernels over $\mathcal X^2, k_1, k_2 : \mathcal{X}^2 \to \mathbb{R}$, $a \in \mathbb R_{+}$ a positive constant, $f$ a real-valued function on $\mathcal X$, $\bm x \in \mathcal{X}, \bm z \in \mathcal{X}$. Then the following functions are Mercer kernels:
\begin{enumerate}
    \item $k(\bm x, \bm z) = k_1(\bm x, \bm z) + k_2(\bm x, \bm z)$, \label{add}
    \item $k(\bm x, \bm z) = ak_1(\bm x, \bm z)$, \label{const_mult}
    \item $k(\bm x, \bm z) = k_1(\bm x,\bm z)k_2(\bm x,\bm z)$, \label{mult_2k}
    \item $k(\bm x, \bm z) = f(\bm x)f(\bm z)$, \label{2func_k}
    \item $k(\bm x, \bm z) = \exp\{k_1(\bm x, \bm z)\}$.\label{exp_k}
\end{enumerate}
\end{proposition}

We give the following property that we will use in the proof that the proposed GC kernel is a Mercer kernel. 
\smallskip

\begin{proposition}\label{K_posDef}
A kernel matrix whose element from the $i$-th row and $j$-th column is given by the inner product of a transformation applied to elements of any finite sequence         $S = \{\bm x_1, \bm x_2, \dots, \bm x_m\}, \bm x_i, \bm x_j \in \mathcal{X} \subset \mathbb{R}^n, \forall i,j \in \{1,\dots,m\}$:
\begin{equation}\label{Product_Phi}
 \bm K(i,j) = \langle \phi(\bm x_i), \phi(\bm x_j) \rangle   
\end{equation}
is positive semi-definite.
\end{proposition}
\begin{proof}
The matrix $\bm K$ is symmetric, so for any non-zero vector $\bm z \in \mathbb{R}^n$ we have: \\\\
$\bm z^T \bm K\bm z= \langle \bm K\bm z,\bm z\rangle= \langle \bm A^T\bm A\bm z,\bm z\rangle = \|\bm A\bm z\|^2 \geq 0$, so it is positive semi-definite. 
\end{proof}

\subsection{Geometry of the SVM}
It has been shown by Amari and Wu in \cites{amariInfGeo, AmariWu1999, WuAmari} that a Riemannian metric is induced by the embedding of the original input space $\mathbb{R}^n$ in a higher dimensional space $\mathbb{R}^m$, where $m \gg n$ and can be infinite. 
This metric can be expressed in terms of the kernel as a matrix $G = (g_{ij})$ with elements given by:
\begin{equation}\label{eq:RiemMetricInducedByK}
    g_{ij}(\bm x) = \pdv*{k(\bm x, \bm y)}{x_i, y_j}\Bigg\rvert_{\bm y=\bm x} 
\end{equation}
It is also known that the volume element at point $\bm x$ is given by 
\begin{equation}\label{magn_factor}
    dV(\bm x) = \sqrt{|g_{ij}(\bm x)|} dx_1\cdots dx_n,
\end{equation}
where $|g_{ij}(x)|$ is the determinant of the Riemannian matrix $(g_{ij})$.
The scaling factor 
$\sqrt{|g_{ij}(x)|}$, called \textit{the magnification factor} in \cite{AmariWu1999}, shows how the volume can be contracted or expanded around a point $\bm x$ under the mapping $\bm \phi$. Since the output function of SVM given by \eqref{OutputFunctionKernel} depends only on the support vectors, Amari and Wu propose to expand the neighborhood around support vectors with the purpose of obtaining a better separation between samples \cites{amariInfGeo, AmariWu1999, WuAmari}. The idea is to modify a given kernel by
\begin{equation}\label{Conformal}
    \Tilde{k}(\bm x, \bm y) = D(\bm x ) D(\bm y ) k(\bm x, \bm y)\qquad,
\end{equation}
called a \textit{conformal transformation} of a kernel, where $D(x)$ is a positive scalar function. 
\begin{proposition}
The kernel given by the conformal transformation from equation \eqref{Conformal} is a Mercer kernel.     
\end{proposition}
\begin{proof}
The proof follows immediately knowing that $k(\bm x, \bm y)$  is a Mercer kernel and applying the Properties \ref{2func_k} and \ref{mult_2k} from Proposition \ref{propertiesK}.
\end{proof}
This transformation leads to a change of the metric to 
\begin{align}\label{metric_induced_by_conf_eq}
\Tilde{g}_{ij}(\bm x) &= D(\bm x)^2g_{ij}(\bm x) + D_i(\bm x ) D_j(\bm x )k(\bm x, \bm x) \nonumber\\
&+ D(\bm x) \{ D_i(\bm x)k_j(\bm x ,\bm x) + D_j(\bm x)k_i(\bm x ,\bm x)\}
\end{align}

where
\begin{equation}\label{first_der_k_eq}
    D_i(\bm x)\overset{\mathrm{not}}{=}\pdv{}{x_i}D(\bm x),\qquad k_i(\bm x,\bm x)\overset{\mathrm{not}}{=}\pdv{}{x_i}k(\bm x,\bm y)\Big\rvert_{\bm y =\bm x}.
\end{equation}
The function $D(\bm x)$ needs to be chosen in such a way that it has large values around the boundaries, and small values far from the boundary. A problem that arises from this is that the location of the boundary is unknown initially. The procedure proposed by Amari and Wu \cites{AmariWu1999, WuAmari} involves a first step of applying the SVM with a kernel to get a first estimation of the support vectors which are known to be mostly located around the boundary, then apply the conformal transformation to the original kernel and apply SVM a second time with the modified kernel. 

Popular choices of kernels for SVM, which we will also use in our experiments, include the Linear kernel given by 
\begin{equation}\label{LinearKernel}
    k_L(\bm x, \bm y) = \langle \bm x , \bm y \rangle.
\end{equation}
and the Gaussian kernel given by
\begin{equation}\label{GaussianKernel}
        k_G(\bm x, \bm y) = e ^ { -\gamma{\| \bm x - \bm y \|^2} }.
\end{equation}

\subsection{Geometry of the Multinomial Distributions}
Consider a family of probability distributions which depends on a finite number of parameters. If the parameters vary smoothly, it can be seen as a parameterized surface \cites{calin, amariMethods}. 
\begin{definition}
Denote a set of probability distributions on that depends on \(n\) parameters \(\theta=\left(\theta_1,\dots,\theta_n\right) \in \mathrm{\Theta} \subset \mathbb{R}^n\) by
\begin{equation}\label{StatisticalFamily}
    \mathcal{S}=\{p_\theta = p(x;\theta)\}.
\end{equation}
\(\mathcal{S}\) is a subset of \(\mathcal{P} (\chi) = \left\{ f:\chi\to\mathbb{R} : f \geq0, \displaystyle \int_\chi f dx = 1\right\}\). If the mapping \(\theta \to p_\theta\) is an immersion, then the set \(\mathcal{S}\) is a \textit{statistical model} of dimension \(n\).
\end{definition}
\begin{definition}
The \textit{log-likelihood function} is given by
\begin{equation}
    \ell(\theta) = \ell(p_\theta)(x) = \ln p_\theta(x).
\end{equation}  
\end{definition}
One of the concepts that plays a central role in information geometry is the \textit{Fisher information metric} which was first introduced by Rao in \cite{Rao}. 
\begin{definition}
The formula for the Fisher information matrix is
\begin{equation}\label{fisher}
        g_{ij}(\theta) = \mathbf{E} \left[\frac{\partial \ell(\theta)}{\partial \theta_i}\cdot \frac{\partial \ell(\theta)}{\partial \theta_j}\right], \quad \forall i,j \in\{1,\dots,n\},
    \end{equation}
where \(\theta = (\theta_1,\dots,\theta_n) \in \mathrm{\Theta}\subset\mathbb{R}^n\) and $\mathbf{E}$ is the expectation of a random variable. 
\end{definition}
It was shown in \cite{calin} that for any statistical model, the Fisher information matrix is a Riemannian metric. As a consequence, the pair \((\mathcal{S},g)\) is a Riemannian manifold.
The Fisher information matrix can also be written as
\begin{equation}\label{fisher2}
    g_{ij}(\theta) = -\mathbf{E}_{\theta}[\partial{\theta_i}\partial{\theta_j}\ell(\theta)]. 
\end{equation}

We present the geometry of the multinomial probability distribution, taken from \cite{calin}. 
Consider $m$ independent, identical trials with $n$ possible outcomes.
Consider the statistical family given by equation \eqref{StatisticalFamily}, with
\begin{equation}
    p(x;\theta) = \frac{m!}{x_1+x_2+\cdots+x_{n+1}}\prod_{i=1}^{n+1}\theta_i^{x_{i}},
\end{equation}
where $\theta = (\theta_i), \theta_i \in (0,1), i=\{1,\dots,n\}$. The parameter space is the $n-$dimensional simplex $\mathcal{P}_n$
\begin{equation}\label{SimplexFormula}
    \mathcal{P}_n = \{\theta\in \mathbb{R}^{n+1}: \sum_{i=1}^{n+1}\theta_i=1,\theta_i>0\}.
\end{equation}
If we consider the transformation $z_i = 2\sqrt{\theta_i}$, then $\sum_{i=1}^{n+1} z_i^2 = 4$
and the statistical manifold of multinomial probability distributions can be identified with the positive portion of the $n$-dimensional sphere of radius $2$:
\begin{equation}
    \mathbb{S}^{n}_{2,+}=\{ z=(z_1,\dots,z_{n+1}) \in \mathbb{R}^{n+1}: \sum_{i=1}^{n+1}z_i^2 = 4 \}.
\end{equation}
Since the transformation given by $z_i$ is an isometry (see \cite{diffusion}), the geodesic distance between two distributions $\theta$ and $\theta'$ on the multinomial can be computed as the shortest curve on $\mathbb{S}^{n}_{2,+}$ connecting $z_{\theta}$ and $z_{\theta'}$. The angle $\alpha$ between unit vectors $\frac{z_{\theta}}{2}$ and $\frac{z_{\theta'}}{2}$ satisfies 
\begin{equation}
    \cos{\alpha}= \left \langle \frac{z_{\theta}}{2}, \frac{z_{\theta'}}{2}\right \rangle.
\end{equation}
Since the distance on the sphere is the product between the radius and the central angle we have from \cites{calin}:
\begin{equation}\label{arccosdist}
    d(\theta, \theta') = 2 \arccos{\left( \sum_{i=1}^{n+1}\sqrt{\theta_i\theta'_i} \right)}.
\end{equation}
We will use this distance in the conformal transformation that we propose in Section \ref{Sec:ProposedMethod}.

\section{Proposed Method}\label{Sec:ProposedMethod}
\subsection{Gaussian Cosine Kernel}
In the following section, we propose a kernel that we will use for classification with SVM. This kernel is an adaptation of the Gaussian kernel considering points lying on the $n-$dimensional sphere, which makes it more suitable for text classification.  
Instead of the Euclidean distance which is used in Gaussian kernel, we propose to use the cosine distance, which is strongly related with the cosine similarity of two vectors \cite{Manning2009}. The cosine similarity is widely used in natural language processing as it provides a measure of the similarity of two documents. 
\begin{definition}
Given two vectors $\bm x, \bm y \in \mathbb{R}^n$, the cosine similarity is defined as:
\begin{equation} \label{cosine_similarity}
    k_{sim}(\bm x , \bm y) = \cossim{x}{y},
\end{equation}
where $\langle,\rangle$ is the inner product and $||\cdot||$ is the $L^2$ norm.  
\end{definition}
\begin{definition}
Given two vectors $\bm x, \bm y \in \mathbb{R}^n$, the cosine distance is defined as:
\begin{equation} \label{cosine_distance}
    k_{dist}(\bm x , \bm y) = \cosdist{x}{y},
\end{equation}
where $\langle,\rangle$ is the inner product and $||\cdot||$ is the $L^2$ norm.  
\end{definition}
We propose the following GC kernel:
\begin{equation}\label{GaussianCosineKernelEq}
    k_{GC}(\bm x, \bm y) = e ^ {-\gamma \left(1 -\frac{\langle \bm x, \bm y\rangle}{\|\bm x\|\|\bm y\|}\right)}, \qquad \gamma \in \mathbb{R_{+}}.
\end{equation}
\begin{proposition}
    The GC kernel is a Mercer kernel.
\end{proposition}
\begin{proof}
    In order to prove this, we will first show that the cosine similarity function, given by Equation \eqref{cosine_similarity} is a Mercer kernel. 
First, we see that the kernel matrix $\bm K$ given by associated to the cosine similarity kernel function is symmetric, from the definition of the kernel function.  
\begin{align*}
\bm K(i,j) = k_{sim}(\bm x_i,\bm x_j) = \frac{\langle \bm x_i, \bm x_j\rangle}{\|\bm x_i\|\|\bm x_j\|} = \frac{\langle \bm x_j, \bm x_i\rangle}{\|\bm x_j\|\|\bm x_i\|} = k_{sim}(\bm x_j, \bm x_i).    
\end{align*}

Next, we need to show that the matrix $\bm K$ is positive semi-definite. 
For this, we will use Proposition \ref{K_posDef}, which means we need to show that the kernel matrix for equation \eqref{cosine_similarity} can be written in the form of equation \eqref{Product_Phi}. Taking any finite sequence         $S = \{\bm x_1, \bm x_2, \dots, \bm x_m\}, \bm x_i, \bm x_j \in \mathcal{X} \subset \mathbb{R}^n, \forall i,j \in \{1,\dots,m\}$, we have:

\begin{align*}
\bm K(i,j) &= \frac{\langle \bm x_i, \bm x_j\rangle}{\|\bm x_i\|\|\bm x_j\|} = \frac{\sum_{k=1}^nx_{ik}x_{jk}}{\|\bm x_i\|\|\bm x_j\|}=\frac{x_{i1}x_{j1}}{\|\bm x_i\|\|\bm x_j\|} + \frac{x_{i2}x_{j2}}{\|\bm x_i\|\|\bm x_j\|} + \cdots +
\frac{x_{in}x_{jn}}{\|\bm x_i\|\|\bm x_j\|}  \\
&= \frac{x_{i1}}{\|\bm x_i\|} \cdot \frac{x_{j1}}{\|\bm x_j\|} + 
\frac{x_{i2}}{\|\bm x_i\|} \cdot \frac{x_{j2}}{\|\bm x_j\|} + \cdots +
\frac{x_{in}}{\|\bm x_i\|} \cdot \frac{x_{jn}}{\|\bm x_j\|}   \\
&=\left \langle \frac{\bm x_i}{\|\bm x_i\|},\frac{\bm x_j}{\|\bm x_j\|} \right \rangle = \langle \phi(\bm x_i), \phi(\bm x_j) \rangle,
\end{align*}
where we made the notation 
$\frac{\bm x_i}{\|\bm x_i\|} \overset{\mathrm{not}} = \left( \frac{x_{i1}}{\|\bm x_i\|}, \frac{x_{i2}}{\|\bm x_i\|}, \cdots \frac{x_{in}}{\|\bm x_i\|}\right)$ and $\phi(\bm x_i) = \frac{x_i}{\|x_i\|}.$
Since we wrote the kernel matrix $\bm K(i,j)$ in the form of \eqref{Product_Phi}, from Proposition \ref{K_posDef} we can conclude that it is a positive semi-definite matrix. Since it satisfies both properties required by Mercer's theorem, the kernel $k_{sim}(\bm x, \bm y)$ is a Mercer kernel.\\
The last part of the proof is a direct consequence of the kernel's properties. We rewrite the GC kernel in a more convenient form:
\begin{align*}
    k_{GC}(\bm x,\bm y) &= e ^{-\gamma(1-k_{sim}(\bm x,\bm y))} =e ^{\gamma(k_{sim}(\bm x, \bm y) - 1)} =\\
    &=e ^{\gamma k_{sim}(\bm x,\bm y)}\cdot\frac{1}{e^{\gamma}} = a\cdot e^{\gamma k_{sim}(\bm x,\bm y)},
\end{align*}
where we made the notation $a = \frac{1}{e^{\gamma}}, a > 0$.
By using properties from Proposition \ref{propertiesK} in the following order, we see that:
\begin{itemize}
    \item From property \ref{const_mult}, we have that $\gamma k_{sim}(\bm x,\bm y)$ is a Mercer kernel.
    \item Then, using property \ref{exp_k}, we have that $e^{\gamma k_{sim}(\bm x,\bm y)}$ is a Mercer kernel.
    \item And finally, by using property \ref{const_mult} again, we have that $a\cdot e^{\gamma k_{sim}(\bm x,\bm y)}$ is a Mercer kernel, which ends the proof.
\end{itemize}
\end{proof}

\subsection{Conformal Transformation}
Following the idea of Amari and Wu in \cites{AmariWu1999, WuAmari} to apply a conformal transformation in order to improve the performance of a given kernel, we propose using a conformal transformation for the GC kernel and analyze its performance for text classification. We also perform experiments with conformal transformation of the Linear and Gaussian kernels given by Equations \eqref{LinearKernel} and \eqref{GaussianKernel}. 

In order to conduct our experiments, we analyzed the conformal transformations proposed by Amari and Wu on synthetic datasets. The first one was proposed by Amari and Wu in 1999 in \cite{AmariWu1999}:
\begin{equation}\label{conf1}
    D_1(\bm x)=\sum_{i\in SV} \alpha_ie^{-\frac{\|\bm x - \bm x_i\|^2}{2\tau^2}},
\end{equation}
where $\tau$ is a free parameter that needs to be adjusted, $\alpha_i$ are given by the output function of SVM \eqref{OutputFunctionKernel} and summation runs over all support vectors. 

In 2002, Wu and Amari proposed an improved version of \ref{conf1} in \cite{WuAmari} that replaces the unique parameter $\tau$ with a set of parameters $\{\tau_i\}$ that accounts for irregular distributions of support vectors. Also, the term $\alpha_i$ was dropped to avoid double counting in the training set, leading to: 
\begin{equation}\label{conf2}
        D_2(\bm x)=\sum_{i\in SV} e^{-\frac{\|\bm x - \bm x_i\|^2}{2\tau_i^2}},
\end{equation} where
\begin{equation}\label{tau_i}
    \tau_i=\frac{1}{M}\sum_{j \in \{M \subset SV\}}\|\bm x_{j} - \bm x_i\|^2.
\end{equation}
The summation in \eqref{tau_i} runs over $M$ support vectors that are nearest to support vector $\bm x_i$. $M$ is a parameter that needs to be adjusted.

In 2005, Williams et al. \cite{williams2005} proposed another version:
\begin{equation}\label{conf3}
    D_3(\bm x)=e^{-\kappa f(\bm x)^2},
\end{equation}
where $f(\bm x)$ is given by \eqref{OutputFunctionKernel}.
We did not consider \ref{conf1}, since \ref{conf2} is already an improved version of it. We analyzed the performance of \ref{conf2} and \ref{conf3} on synthetic datasets separated by nonlinear boundaries. We generated datasets in the region $[-0.5, 0.5] \times [-0.5, 0.5]$ with the boundary given by $y=0.5\sin(2\pi x)$ as described in \cite{WuAmari} and another set in  $[-1, 1] \times [-1, 1]$  with the boundary given by 

$y = 2e^{-4x^2}-1$ as described in \cite{williams2005}. We aimed to reproduce the results described in the mentioned articles. We performed $10000$ experiments for each of the proposed transformations, each time using a different training set of $100$ samples and a test set of $10000$ samples for \eqref{conf2} and $1000$ samples for \eqref{conf3}, generated randomly and uniformly distributed, as described in the experiments done in \cites{WuAmari, williams2005}. 
The performance improvement for the modified kernel is computed by comparing the test error for the original model $E_o$ vs. the modified one $E_m$, averaged over $10000$ experiments. The error decrease is expressed in percentage and is computed as 
\begin{equation}\label{ErrDecrease}
E_d = \frac{E_o - E_m}{E_o}\cdot 100     
\end{equation}
Negative values indicate a decrease in performance. The conformal transformation \eqref{conf2} showed superior results over transformation \eqref{conf3} as it can be seen in Table \ref{tab:ConfAmari} compared to Table \ref{tab:ConfWilliams}, so we considered it for application to text classification. 

\begin{table}[]
\centering
\begin{tabular}{|c|c|}
\hline
\textbf{$\sigma$} & \textbf{\begin{tabular}[c]{@{}c@{}}Error decrease\\ (\%)\end{tabular}} \\ \hline
0.05 & 17.817   \\ \hline
0.3  & -17.761  \\ \hline
0.4  & -65.65   \\ \hline
0.5  & -106.054 \\ \hline
0.7  & -60.005  \\ \hline
1    & -28.485  \\ \hline
2    & -28.504  \\ \hline
\end{tabular}
\caption{Result of SVM with Gaussian kernel for different values of $\sigma$ using conformal transformation given by \eqref{conf3}.} 
\label{tab:ConfWilliams}
\end{table}

\begin{table}[]
\centering
\begin{tabular}{|c|c|}
\hline
\textbf{$\sigma$} & \textbf{\begin{tabular}[c]{@{}c@{}}Error decrease\\ (\%)\end{tabular}} \\ \hline
0.05 & -9.117 \\ \hline
0.3  & 7.707  \\ \hline
0.4  & 6.646  \\ \hline
0.5  & 7.626  \\ \hline
2    & 15.528 \\ \hline
\end{tabular}
\caption{Result of SVM with Gaussian kernel for different values of $\sigma$ using conformal transformation given by \eqref{conf2}.}
\label{tab:ConfAmari}
\end{table}

However, the conformal transformation proposed in \cite{WuAmari}
uses the Euclidean distance, which is inappropriate for text documents due to the high dimensionality and sparsity of data \cite{huang2008similarity}. We want to use a distance that has the properties required by the conformal transformation that it should be large only around the support vectors and should decrease with the distance to the support vectors.  

We propose the following conformal transformation:

\begin{equation}\label{ConfTransCosine}
 D_{cos}(\bm x)=\DSumCos{s}{x}{x_s} ,\quad
 \end{equation}
with
\begin{equation}\label{TauCosine}
    \tau_s=\frac{1}{M}\sum_{j \in \{M \subset SV\}}\left(\cosdist{x_j}{x_s}\right),
\end{equation}

where $s$ in \eqref{ConfTransCosine} runs over all support vectors $\bm x_s$ and the summation in \eqref{TauCosine} goes over $M$ support vectors which are closest to support vector $\bm x_s$. 
We did not use square the distances in the GC kernel, nor in the conformal transformation. The reason is that the cosine distance is based on the orientation of vectors and not on magnitude. If we would use the square cosine distance, points that are opposed (dissimilar) to a given point would be treated the same as points that are very similar. This is not desired for text classification. 

For  concise expressions, we will make the following notations for two $n$-dimensional vectors $\bm x, \bm y \in \mathcal{X} \subset\mathbb{R}^n$
\begin{align*}
    \sqrt{\bm x}\qquad   &\overset{\mathrm{not}} = \left(\sqrt{x_1}, \sqrt{x_2},\cdots,\sqrt{x_n}\right),  \\
\inprod{\sqrt{x}}{\sqrt{y}} &\overset{\mathrm{not}}= \sum_{i=1}^n{\sqrt{x_iy_i}}.
\end{align*}
For the GC kernel, in addition to \eqref{cosine_distance} we also consider the geodesic distance on the sphere given by Equation \eqref{arccosdist}, so we have the conformal transformation:

\begin{equation}\label{ConfTransArccos}
 D_{arc}(\bm x)=\DSumArccos{s}{x}{\bm x_s},
\end{equation} 

with

\begin{equation}\label{TauArccos}
    \tau_s=\frac{1}{M}\sum_{j \in SV}2 \arccos{ \left\langle \sqrt{\bm x_j},\sqrt{\bm x_s} \right\rangle}.
\end{equation}
Similar to the previous transformation, $s$ in \eqref{ConfTransArccos} runs overall support vectors $\bm x_s$ and the summation in \eqref{TauArccos}
goes over $M$ support vectors which are closest to support vector $\bm x_s$. $M$ is a free parameter that needs to be adjusted. 

\subsection{Induced Riemannian Metric}
Next, we will calculate the Riemannian Metric induced by GC kernel. We recall that this can be calculated with the formula:
\begin{equation*}
  g_{ij}(\bm x) = \pdv*{k(\bm x, \bm y)}{x_i, y_j}\Bigg\rvert_{\bm y=\bm x}   .
\end{equation*}
We make a note here that we also considered the Diffusion kernel introduced in \cite{diffusion} on the multinomial manifolds for our experiments, given by:
\begin{equation}\label{DiffusionKernel}
    K_t(\bm x, \bm y)=(4\pi t)^{-\frac{n}{2}}\exp{\left(-\frac{1}{t}\arccos^2\left(\sum_{i=1}^{n+1}\sqrt{\bm x_i\bm y_i}\right)\right)}.
\end{equation}

We notice, however, that this kernel does not induce a Riemannian metric. Due to the derivative of the $\arccos$ function and the fact that our data lies on the multinomial and satisfies Equation \eqref{SimplexFormula}, the second derivative of the kernel \eqref{InducedRiemMetric} cannot be evaluated at the point $\bm y = \bm x$. \\
Therefore, we consider the GC kernel more appropriate for this specific type of data lying on the multinomial manifold. 

We compute the first partial derivative of the GC kernel
$k(\bm x,\bm y) = e^{-\gamma\left(1-\frac{\langle \bm x, \bm y \rangle}{\|\bm x\|\|\bm y\|}\right)}$.
\begin{align}\label{First_der_GaussCosKernel}
    \pdv{k(\bm x,\bm y)}{x_i} &= \gamma e^{-\gamma(1-\frac{\langle \bm x, \bm y\rangle}{\|\bm x\|\|\bm y\|})}\frac{y_i\|\bm x\|\|\bm y\| - \sum_i^nx_iy_i\|\bm y\|\frac{2x_i}{2\sqrt{\sum_i^nx_i^2}}}{\|\bm x\|^2\|\bm y\|^2} \nonumber \\ 
    &= \gamma k(\bm x, \bm y)\frac{x_i\|\bm x\|^2\|\bm y\| - x_i\|\bm y\|\inprod{x}{y}}{\|\bm x\|^3\|\bm y\|^2} \nonumber \\
    &= \gamma k(\bm x, \bm y) \frac{y_i\|x\|^2 - x_i\inprod{x}{y}}{\|\bm x\|^3\|\bm y\|}
\end{align}
For computing the second partial derivative, we first do a preliminary computation and we see that:
\begin{align*}
    \pdv*{}{y_j}\left( \frac{y_i\norm{x}^2 - x_i\inprod{x}{y}}{\norm{x}^3\cdot\norm{y}} \right) = \pdv*{}{y_j}\left( \frac{y_i}{\norm{x}\norm{y}} - \frac{x_i\inprod{x}{y}}{\norm{x}^3\norm{y}} \right) 
\end{align*}
with
\begin{align*}
    \pdv*{}{y_j}\left( \frac{y_i}{\norm{x}\norm{y}} \right) &= \frac{1}{\sqrt{\sum_ix_i^2}}\cdot\frac{\delta_{ij}\sqrt{\sum_iy_i^2} - y_i\frac{1}{2\sqrt{\sum_iy_i^2}}\cdot2y_j}{\sum_iy_i^2}\\
    &=\frac{1}{\norm{x}}\cdot\frac{\delta_{ij}\norm{y} - \frac{y_iy_j}{\norm{y}}}{\norm{y}} = \frac{1}{\norm{x}}\left( \delta_{ij} - \frac{y_iy_j}{\norm{y}^2} \right) 
\end{align*}
and
\begin{align*}
    \pdv*{}{y_j}\left( \frac{x_i\inprod{x}{y}}{\norm{x}^3\norm{y}} \right) &= \frac{x_i}{\norm{x}^3}\cdot\pdv*{}{y_j}\left( \frac{\sum_ix_iy_i}{\sqrt{\sum_iy_i^2}} \right) \\ &= \frac{x_i}{\norm{x}^3}\cdot\frac{x_j\sqrt{\sum_iy_i^2} - \sum_ix_iy_i\frac{1}{2\sqrt{\sum_iy_i^2}}\cdot2y_j}{\sum_iy_i^2} \\
    &=\frac{x_i}{\norm{x}^3}\frac{x_j\norm{y}^2 - y_j\inprod{x}{y}}{\norm{y}^3}.
\end{align*}
We compute the second mixed partial derivative:
\begin{align*}
\pdv{k(\bm x , \bm y)}{x_i,y_j} = \gamma &\Bigg \{ \gamma k(\bm x, \bm y) \frac{x_j\|\bm y\|^2 - y_j\langle \bm x, \bm y \rangle}{\|\bm y\|^3\|\bm x\|} \cdot \frac{y_i\norm{x}^2 - x_i\langle \bm x,\bm y \rangle}{\|\bm x\|^3\|\bm y\|} +\\
&+ k(\bm x, \bm y) \Big [\frac{1}{\|\bm x\|}\left (\delta_{ij} -\frac{y_iy_j}{\|\bm y\|^2}\right) - \frac{x_i}{\|\bm x\|^3}\cdot\frac{x_j\|\bm y\|^2 - y_j\langle\bm x, \bm y\rangle}{\norm{y}^3} \Big] \Bigg\}    
\end{align*}
Now we need to evaluate the second derivative at $\bm y = \bm x$.
We first notice that:
\begin{equation}\label{kofxx=1}
k(\bm x, \bm x)= e^{-\gamma(1 - \frac{\langle \bm x, \bm x\rangle}{\|\bm x\|^2})} = e^0 = 1.    
\end{equation}
This leads to:
\begin{align}\label{InducedRiemMetric}
g_{ij}(\bm x) = \pdv{k(\bm x , \bm y)}{x_i,y_j}\Bigg\rvert_{\bm y=\bm x} = \frac{\gamma}{\|\bm x\|}\left( \delta_{ij} - \frac{x_ix_j}{\|\bm x\|^2} \right) .     
\end{align}
\subsection{Metric Induced by the Conformal Transformation}
In order to compute the metric induced by the conformal transformation, we first notice from Equation \eqref{First_der_GaussCosKernel} that: 
\begin{align}\label{first_der_k_0}
   k_i(\bm x,\bm x)= 0.
\end{align}
By replacing \eqref{first_der_k_0} and \eqref{kofxx=1} in \eqref{metric_induced_by_conf_eq} we have:
\begin{align}
    \Tilde{g}_{ij}(\bm x) &= D(\bm x)^2g_{ij}(\bm x) + D_i(\bm x ) D_j(\bm x ).
\end{align}
For the GC kernel, using the conformal transformation given by Equation \eqref{ConfTransCosine}, we have:
\begin{align}
    D_i(\bm x) = \pdv{D_{cos}(\bm x)}{x_i} = \DSumCos{s}{x}{x_s} \cdot \frac{x_s\norm{\bm x}^2 - x_i\inprod{x}{x_s}}{\norm{x}^3\norm{x_s}}.
\end{align}
Then the metric induced by the conformal transformation becomes:
{
\begin{align*}
    \Tilde{g}_{ij}(\bm x) &=D_{cos}(\bm x)^2\metric{x}{i}{j}+\\
    &+ \dercossim{x}{x_s}{i}{s}\cdot \dercossim{x}{x_s}{j}{s} = \\
    &= \left (\DSumCos{s}{x}{x_s}\right )^2\metric{x}{i}{j}+\\
    &+ \left(\dercossim{x}{x_s}{i}{s}\right)\cdot \left(\dercossim{x}{x_s}{j}{s}\right).
\end{align*}
}
If we consider the conformal transformation given by \eqref{ConfTransArccos}, we have:
\begin{align*}
     D_i(\bm x) = \pdv{D_{arc}(\bm x)}{x_i}= \derarccos{x}{x_s}{i}{s}
\end{align*}
and 
{
\begin{align*}
    &\Tilde{g}_{ij}(\bm x) =D_{arc}(\bm x)^2\metric{x}{i}{j}+\\
    +&\derarccos{x}{x_s}{i}{s} \cdot \\
     &\cdot\derarccos{x}{x_s}{j}{s} =\\
    &= \left( \DSumArccos{s}{x}{x_s}\right)^2\metric{x}{i}{j} +\\
    &+ \left(\derarccos{x}{x_s}{i}{s}\right)\cdot \\ &\cdot\left( \derarccos{x}{x_s}{j}{s}\right).
\end{align*}
}
Using the formula for the metric induced by the conformal transformation, one can compute the magnification factor given by \eqref{magn_factor}. By comparing the magnification factor for different kernels, one could choose the kernel that yields the highest value, hence leading to an optimal separation between classes. However, due to the extensive nature of these computations, we have omitted them in this discussion. From a practical perspective, considering the high-dimensionality of the analyzed dataset, it is more efficient to implement and apply the conformal transformation of the kernels and assess their performance rather than computing the magnification factor upfront. 

\subsection{Experimental Procedure}\label{subsec:Algorithm}
These are the three kernels that are used in the current experiments:
\begin{itemize}
    \item Linear kernel, given by Equation \eqref{LinearKernel}.
    \item Gaussian kernel, given by Equation \eqref{GaussianKernel}. 
    \item GC kernel, given by Equation \eqref{GaussianCosineKernelEq}
\end{itemize}
The experimental procedure we propose is based on a slightly modified version of \cite{WuAmari}, where we dropped the iterative approach.
\begin{enumerate}
    \item Apply SVM using the original kernel in order to extract an initial information about the support vectors. In this experiment, we used the Linear, Gaussian, and GC kernels, defined by equations \eqref{LinearKernel}, \eqref{GaussianKernel}, \eqref{GaussianCosineKernelEq}.
    \item  Modify the kernel using the conformal transformation given by Equation \eqref{ConfTransCosine}. In addition, for the GC kernel, conformal transformation given by \eqref{ConfTransArccos} was used. The value for $M$ in \eqref{TauCosine} and \eqref{TauArccos} was chosen to be $M = 3$, a value that was empirically validated by Wu and Amari in \cite{WuAmari}. 
    \item Apply SVM with the modified kernel. 
\end{enumerate}

\section{Experiments}\label{Sec:Experiments}

\subsection{Text Documents Representation}
In this section, we present the application of conformal transformation of the Linear, Gaussian, and GC kernels to text document classification. For representing documents we consider the "bag of words" approach \cites{joachims1997BagOfWords, Zhang2010BagOfWords}, a technique often used in natural language processing and information retrieval. This implies that word order and grammatical structure are discarded and the focus is on word occurrence. This is a simple way of transforming text into numerical data for processing purposes. Assuming a vocabulary $V$ of size $n+1$, a document is represented as a sequence of words over the vocabulary. Each document can be seen as an outcome of a multinomial distribution \cites{MultinomialEventModel, diffusion}. When we say a document is a draw from the multinomial distribution, we mean that the words in the document are drawn independently from a distribution over the vocabulary, where the probability of drawing each word is proportional to its frequency in the document. 

We denote by $x_w$ the number of times the word $w$ appears in the document. Then $\{x_w\}_{w\in V}$ is the sample space of the multinomial distribution. As indicated in \cite{diffusion}, we can embed documents in the multinomial simplex using an embedding function $\hat{\theta}: \mathbb{Z}_{+}^{n+1} \to \mathcal{P}_n$. We consider some of the most common embeddings used in text processing and information retrieval \cite{salton1988termTFIDF}, the \textit{term frequency} (TF) and \textit{term frequency-inverse document frequency} (TFIDF). 
The term frequency of a term $w$ in a document $d$ is defined as the number of times $w$ appears in $d$, which we denoted by $x_w$.
The document frequency $df_w$ of a term $w$ is defined as the number of documents in which the term $w$ appears at least once. 
We denote by $N$ the total number of documents in the corpus. The inverse document frequency $idf$ is $N/df$. For a large corpus of documents, this value can become very large, so a log version is used. In practical applications, like the one we use in this experiment, smoothing is applied and division by zero is avoided by considering (see \cite{scikit-learn}):
\begin{align}
    idf_w = \log{\frac{N+1}{df_w+1}} + 1.
\end{align}

The \textit{term frequency-inverse document frequency} (TFIDF) is defined as:
\begin{equation}
    tfidf_w = tf_w \cdot idf_w = x_w (\log{\frac{N+1}{df_w+1}} + 1).
\end{equation}

In text classification, the TF and TFIDF representations are known to be normalized to unit length using $L^1$ or $L^2$ norm. 
When viewed in the context of the multinomial distribution, the observed term frequencies in the TF or TFIDF representation can be used to estimate the probabilities of different terms occurring within the document. 
The embedding of documents using $L^1$ normalization is: 
\begin{equation}
    \hat{\Theta}_{tf}(\bm x) = \left( \frac{x_1}{\sum_ix_i},\dots, \frac{x_{n+1}}{\sum_ix_i}\right).
\end{equation}
\begin{equation}
    \hat{\Theta}_{tfidf}(\bm x) = \left( \frac{x_1idf_1}{\sum_ix_iidf_i},\dots, \frac{x_{n+1}idf_{n+1}}{\sum_ix_iidf_i}\right).
\end{equation}

Similar formulas can be given using the $L^2$ norm. When using the GC kernel, only $L^1$ norm will be used in order to obtain a valid embedding in the simplex and ensure that the sum of term frequencies adds up to one, which is a fundamental property of a probability distribution. We chose to only use $L^1$ such that we can consider the mapping of the multinomial to the sphere as presented in Section \ref{Sec:Preliminaries} and use the geodesic distance on the sphere in the conformal transformation \eqref{ConfTransArccos}.
For the Linear and Gaussian kernels, both $L^1$ and $L^2$ will be used in the experiments. 
We analyze the performance of the SVM algorithm using the original and modified kernel for each of the three kernels. This comparison is quantified using the F1 score, a harmonic mean of precision and recall, providing a robust metric for model performance evaluation (see details in Subsection \ref{subsec:Methodology}).

\subsection{Dataset}
In our experiments, we used the Reuters-21578 dataset \footnote{Available at \url{http://www.nltk.org/nltk_data}}, which has become a benchmark in text classification tasks. This is a collection of $10788$ news documents, having $1.3$ million words in total. They are categorized into more than $135$ topics, but some of the topics contain less than $100$ documents. In this experiment we chose only the topics with more than $500$ documents, as indicated in \cite{diffusion}. Those topics are "earn", "acq", "money-Fx", "grain" and "crude" having $3964$, $2369$, $717$, $582$ and $578$ documents, respectively.

This dataset is available through the Natural Language Toolkit (NLTK), a platform designed for building Python programs to work with human language data. It offers easy to use interfaces to $50$ text corpora along with libraries for text parsing, stemming, processing and classification \footnote{ \url{https://www.nltk.org/}}.

We conducted two binary classification experiments. 
For the first experiment we labeled one topic as the "positive" topic and all the remaining topics as "negative", with the intention of distinguishing one topic against the others. 
For the second experiment we labeled one topic as "positive" and chose another topic as the "negative" example, and we did this for all pairs of two topics. We did not reverse the classes, meaning if we chose topic $A$ to be positive and $B$ to be negative, we did not repeat the experiment with $A$ as negative and $B$ as positive. 

\subsection{Data preparation}
The role of data preparation in a classification process is crucial for ensuring the quality and efficiency of the results. Therefore, before training the SVM classifiers, we applied some common preprocessing steps on the documents in order to reduce the dimensionality of data and improve classification accuracy \cite{TextProcessing}. The manner in which preprocessing is conducted is intricately influenced by the specific attributes, content, and structure of the dataset under consideration. The following steps are suitable for the Reuters dataset and help remove noise and irrelevant information from the dataset: 
\begin{itemize}
    \item We converted all words to lowercase.
    \item We removed stop-words like: "a", "and", "the". These are words that have a high frequency in documents and do not bring much information in classification tasks. 
    \item We considered valid words those that contain only letters, so we removed all words that contain other characters. 
    \item We removed all words smaller than three letters. Words like "to", "be", "on" are common in most of the documents in our dataset and bring no useful information for the classification task.
    \item We applied stemming, a common technique in text processing, that removed suffixes or prefixes from words, keeping only their root form. We used PortStemmer as stemming tool \cite{willett2006PortStem}, offered by NLTK\footnote{ \url{https://www.nltk.org/howto/stem.html}}. 
    \item We discarded all documents that became empty after the above steps.
\end{itemize}

After this data preparation step we ended up with $8167$ documents and a dictionary containing more than $16000$ words. 

\subsection{Methodology}\label{subsec:Methodology}
The setup for this experiment was inspired by \cite{diffusion}.
In this article, we employed the Support Vector Machine (SVM) algorithm for the classification task. The SVM model was implemented using the \texttt{scikit-learn} library \cite{scikit-learn}, a popular machine learning library in Python. 
In our study, we split the dataset into training and test subsets using a stratified $k$-fold cross-validation strategy, where we empirically chose $k=20$, as suggested in \cite{diffusion}. This strategy ensured that the ratio of positive to negative instances was preserved across all folds in both the training and test sets.

In our experiments we had to tune several hyper-parameters: the soft-margin $C$ parameter had to be adjusted in all experiments, since it is a parameter of the SVM algorithm that controls the trade-off between the training error and the margin. We tested $C$ in the set $\{1, 10, 100 ,1000\}$, keeping the same value for the original and modified kernel in order to make sure we only measure the influence of the conformal transformation. 
For the Gaussian and GC kernels we tested the $\gamma$ parameter in the set $\{0.0001, 0.001, 0.01, 0.1\}$. For all three kernels we tested both TF and TFIDF representations. We used $L^1$ and $L^2$ norms for the Linear and Gaussian kernels. For the GC kernel we tested only the $L^1$ norm, since only this norm ensures a valid embedding into the probability simplex as mentioned in \cite{diffusion}. 

We used the \texttt{TfidfVectorizer} from the \texttt{skikit-learn} library with the \texttt{use\_idf} parameter on \texttt{false}, respectively on \texttt{true} for obtaining the TF and TFIDF representation. \texttt{TfidfVectorizer} also does normalization and allows specifying the norm. We left all other parameters of \texttt{TfidfVectorizer} at their default settings.

We evaluate a model's performance by looking at the F1 score, rather than the accuracy, since there is an imbalance in the distribution of positive and negative examples in the dataset \cite{goutte2005F1}. The F1 score is an adequate measure in this case. Labeling all examples as negative would lead to a high accuracy score, which would distort the perception of the model's performance. We recall the formula for the F1 score:
\begin{equation}
    F1=\frac{2P\cdot R}{P+R},
\end{equation}
where 
\begin{align}
    P = \frac{{TP}}{{TP+FP}}, \quad  R = \frac{TP}{TP+FN},
\end{align}
where $P$ is the precision, $R$ is the recall, $TP$ represents the number of true positive samples and $FP$ (resp. $FN$) represents the number of false positive (resp. false negative) samples identified by the model. The F1 score is computed as the mean value over the $20$-fold cross-validation. 

We have also documented the $p-$value statistic for the F1 score to emphasize where the change in performance is statistically significant according to the paired $t$-test at the $0.05$ level. If the $p-$value is less than $0.05$, we consider that the improvement in accuracy is statistically significant.
Another interesting property to notice is the number of support vectors in the original versus the modified kernel model. A smaller number of support vectors leads to simpler models and faster prediction times \cite{NrSupVec}. A simpler model is also known to generalize better and reduce the chances of over-fitting. We consider an improvement the cases where the decrease F1 is not statistically significant but we see a decrease in the number of support vectors. This type of compromise between accuracy and model complexity can be acceptable in some scenarios. This combination of metrics suggests a nuanced evaluation where the overall impact on model effectiveness needs to be carefully assessed, taking into account factors such as computational efficiency, interpretability, and the specific requirements of the problem domain.

\section{Results Analysis and Interpretation}\label{Sec:Results}
Using the subset of the Reuters dataset and the combination of parameters described in Section \ref{Sec:Experiments} we present the results for one-vs.-rest and the one-vs.-one binary classification tasks. We compare the results of SVM using the Linear, Gaussian and GC kernels with the modified kernels obtained by transformations D$_{cos}$ \eqref{ConfTransCosine} and D$_{arc}$ \eqref{ConfTransArccos} using the algorithm described in the Section \ref{subsec:Algorithm}.
We consider an improvement in the following cases:
\begin{itemize}
    \item F1 metric has a higher value for the modified model compared to the original model.
    \item F1 metric for the modified kernel has a smaller value compared to the original kernel, but it is statistically non-significant ($p-$value $\geq 0.05$). Additionally, the number of support vectors in the modified model is smaller compared to the original model. 
\end{itemize}
\subsection{One vs. Rest Task}
For this task, we compared the performance of the original Linear, Gaussian and GC kernels with their modified counterparts by the conformal transformations D$_{cos}$ \eqref{ConfTransCosine} and D$_{arc}$ \eqref{ConfTransArccos} for GC only, when employed in the SVM algorithm. 
For the Linear kernel \eqref{LinearKernel}, compared to the modified kernel obtained by transformation D$_{cos}$ \eqref{ConfTransCosine},
we notice that the procedure is especially successful in the cases where the original kernel does not perform well. For this kernel, we noticed an increase in accuracy, measured by the F1 metric, in $38\%$ of the studied scenarios. If we also consider the cases where the F1 metric decreased, but not significantly ($p-$value $\geq 0.05$), and the number of support vectors is reduced, we see an improvement in $61\%$ of the cases. 
For the Gaussian kernel \eqref{GaussianKernel}, compared to the modified kernel obtained by transformation $D_{cos}$ \eqref{ConfTransCosine}, we notice an increase in accuracy in $71\%$ of the tested scenarios. If we consider also the cases where the F1 metric decreased but not significantly, and the number of support vectors is reduced, we see an improvement in $84\%$ of the cases. 

For the GC kernel \eqref{GaussianCosineKernelEq} we only considered $L^1$ normalization in order to obtain a valid embedding in the simplex and to compare the results using the conformal transformations given by Equations \eqref{ConfTransCosine} and \eqref{ConfTransArccos}. For both conformal transformations, in $62.5\%$ of the tested scenarios, we noticed an increase in accuracy, measured by the F1 metric. If we also consider the cases where the F1 metric decreased but not significantly, and the number of support vectors is reduced, we see an improvement in $78\%$ of the cases for \eqref{ConfTransCosine} and $75\%$ of the cases for \eqref{ConfTransArccos}. These findings are summarized in Table \ref{tab:OVR_SummaryPercentImprov}. In the Accuracy Increase column, we present the percentage of cases where an increase in accuracy using the F1 metric was observed. In the Efficiency Increase column besides an increase in F1 metric we additionally consider the cases where a non-significant decrease in F1 was observed accompanied by a decrease in the number of support vectors.
\begin{table}[]
\centering
\begin{tabular}{|c|c|c|}
\hline
\textbf{Model} & \textbf{\begin{tabular}[c]{@{}c@{}}Accuracy\\ Increase\\ (\% of tested\\ scenarios)\end{tabular}} & \textbf{\begin{tabular}[c]{@{}c@{}}Efficiency\\ Increase\\  (\% of tested \\ scenarios)\end{tabular}} \\ \hline
\begin{tabular}[c]{@{}c@{}}Linear\\ with D$_{cos}$\end{tabular} & 38\% & 61\% \\ \hline
\begin{tabular}[c]{@{}c@{}}Gauss\\ with D$_{cos}$\end{tabular} & \textbf{72\%} & \textbf{84\%} \\ \hline
\begin{tabular}[c]{@{}c@{}}GC\\ with D$_{cos}$\end{tabular} & 62.5\% & 78\% \\ \hline
\begin{tabular}[c]{@{}c@{}}GC\\ with D$_{arc}$\end{tabular} & 62.5\% & 75\% \\ \hline
\end{tabular}
\caption{Percentage of tested scenarios where an improvement was observed in the conformally transformed kernel compared to the original kernel.}
\label{tab:OVR_SummaryPercentImprov}
\end{table}

\textbf{Discussion and Implications}

For the one-vs.-rest task, the conformal transformation technique proves effective for kernels with poor performance. The technique has the biggest impact on the Gaussian kernel, improving the performance in $72\%$ of the tested scenarios by increasing accuracy or $84\%$ of the cases by increasing model efficiency.

If we consider the original kernels' best accuracy (maximum F1 values), we see that applying the conformal transformation did not lead to further improvement on any of the kernels. This is a contradiction with the original experiments performed by Amari and Wu in \cites{AmariWu1999,WuAmari}, where they first selected the best original model and then applied the conformal transformation on it which led to further improvement. A possible explanation for this might be the different dataset properties. In \cites{AmariWu1999,WuAmari} the authors used a balanced dataset for training. In \cites{WuAmari} we see the number of training samples ($100$) was less than the number of test samples ($10000$), while in the current experiment, the number of training samples is $19$ times higher than the number of test samples, due to the chosen stratified $k-$fold validation technique.  We present these results in Table \ref{tab:OVR_BestF1orig}. These results are the maximum values obtained for the F1 metric for each of the original kernels. These results were obtained using the TF representation. The conformal transformation could not further improve the F1 metric, it only reduced the number of support vectors. However, the decrease in accuracy 
in these cases is statistically significant, indicated by the $p-$values, which are $< 0.05$. 
\begin{table}[]
\centering
\resizebox{\textwidth}{!}{%
\begin{tabular}{|c|c|c|c|c|cc|cc|c|}
\hline
\multirow{2}{*}{\textbf{Model}}                                 & \multirow{2}{*}{\textbf{Task}}                              & \multirow{2}{*}{\textbf{Norm}} & \multirow{2}{*}{\textbf{$\gamma$}} & \multirow{2}{*}{\textbf{C}} & \multicolumn{2}{c|}{\textbf{F1}}                         & \multicolumn{2}{c|}{\textbf{\#S.V.}}                     & \multirow{2}{*}{\textbf{\begin{tabular}[c]{@{}c@{}}P-value\\  for F1\end{tabular}}} \\ \cline{6-9}
                                                                &                                                             &                                &                                    &                             & \multicolumn{1}{c|}{\textbf{original}} & \textbf{custom} & \multicolumn{1}{c|}{\textbf{original}} & \textbf{custom} &                                                                                     \\ \hline
\begin{tabular}[c]{@{}c@{}}Linear\\ orig. vs. D$_{cos}$\end{tabular}   & \begin{tabular}[c]{@{}c@{}}earn\\ vs. rest\end{tabular}     & $L^2$                             & -                                  & 1                           & \multicolumn{1}{c|}{\textbf{0.98340}}  & 0.97732         & \multicolumn{1}{c|}{876}               & \textbf{820}    & \textbf{0.018809}                                                                   \\ \hline
\begin{tabular}[c]{@{}c@{}}Gaussian\\ orig. vs. D$_{cos}$\end{tabular} & \begin{tabular}[c]{@{}c@{}}money-fx\\ vs. rest\end{tabular} & $L^2$                             & 0.001                              & 1000                        & \multicolumn{1}{c|}{\textbf{0.97706}}  & 0.96351         & \multicolumn{1}{c|}{416}               & \textbf{378}    & \textbf{0.000201}                                                                   \\ \hline
\begin{tabular}[c]{@{}c@{}}GC\\ orig. vs. D$_{cos}$\end{tabular}       & \begin{tabular}[c]{@{}c@{}}earn\\ vs. rest\end{tabular}     & $L^1$                             & 0.001                              & 1000                        & \multicolumn{1}{c|}{\textbf{0.98340}}  & 0.97757         & \multicolumn{1}{c|}{875}               & \textbf{819}    & \textbf{0.020046}                                                                   \\ \hline
\begin{tabular}[c]{@{}c@{}}GC\\ orig. vs. D$_{arc}$\end{tabular}       & \begin{tabular}[c]{@{}c@{}}earn\\ vs. rest\end{tabular}     & $L^1$                             & 0.001                              & 1000                        & \multicolumn{1}{c|}{\textbf{0.98340}}  & 0.97722         & \multicolumn{1}{c|}{875}               & \textbf{819}    & \textbf{0.014901}                                                                   \\ \hline
\end{tabular}
}
\caption{Maximum values obtained for the F1 metric for each of the original kernels.}
\label{tab:OVR_BestF1orig}
\end{table}

Analyzing the best accuracy of all transformed models, we notice that for $3$ out of $7$ models, we have an increase in F1 compared to the original kernels. In $4$ out of $7$ models we have a decrease in F1 which is not statistically significant. In $2$ out of the latter $4$, we see a decrease in the number of support vectors. These results are summarized in Table \ref{tab:OVR_BestF1custom}. Results from this table were obtained using the TF representation and $L^1$ norm. The $p-$values which are less than $0.05$ are marked in bold, showing the cases where the change in the F1 metric is statistically significant.
\begin{table}[]
\centering
\resizebox{\textwidth}{!}{%
\begin{tabular}{|c|c|c|c|cc|cc|c|}
\hline
\multirow{2}{*}{\textbf{Model}} & \multirow{2}{*}{\textbf{Task}} & \multirow{2}{*}{\textbf{$\gamma$}} & \multirow{2}{*}{\textbf{C}} & \multicolumn{2}{c|}{\textbf{F1}} & \multicolumn{2}{c|}{\textbf{\#S.V.}} & \multirow{2}{*}{\textbf{\begin{tabular}[c]{@{}c@{}}P-value\\ for F1\end{tabular}}} \\ \cline{5-8}
 &  &  &  & \multicolumn{1}{c|}{\textbf{original}} & \textbf{custom} & \multicolumn{1}{c|}{\textbf{original}} & \textbf{custom} &  \\ \hline
\begin{tabular}[c]{@{}c@{}}Linear\\ orig. vs. D$_{cos}$\end{tabular} & \begin{tabular}[c]{@{}c@{}}earn\\ vs. rest\end{tabular} & - & 10 & \multicolumn{1}{c|}{0.97555} & \textbf{0.97751} & \multicolumn{1}{c|}{1081} & \textbf{775} & 0.40255 \\ \hline
\begin{tabular}[c]{@{}c@{}}Gaussian\\ orig. vs. D$_{cos}$\end{tabular} & \begin{tabular}[c]{@{}c@{}}grain\\ vs. rest\end{tabular} & 0.0001 & 10 & \multicolumn{1}{c|}{0.00000} & \textbf{0.98101} & \multicolumn{1}{c|}{1062} & \textbf{388} & \textbf{1.98E-32} \\ \hline
\begin{tabular}[c]{@{}c@{}}GC\\ orig. vs. D$_{cos}$\end{tabular} & \begin{tabular}[c]{@{}c@{}}earn\\ vs. rest\end{tabular} & 0.1 & 1 & \multicolumn{1}{c|}{0.96925} & \textbf{0.97874} & \multicolumn{1}{c|}{1716} & \textbf{872} & \textbf{0.001473} \\ \hline
\begin{tabular}[c]{@{}c@{}}GC \\ orig. vs. D$_{cos}$\end{tabular} & \begin{tabular}[c]{@{}c@{}}earn\\ vs. rest\end{tabular} & 0.1 & 10 & \multicolumn{1}{c|}{\textbf{0.98263}} & 0.97874 & \multicolumn{1}{c|}{903} & \textbf{879} & 0.084303 \\ \hline
\begin{tabular}[c]{@{}c@{}}GC\\ orig. vs. D$_{cos}$\end{tabular} & \begin{tabular}[c]{@{}c@{}}earn\\ vs. rest\end{tabular} & 0.1 & 100 & \multicolumn{1}{c|}{\textbf{0.97980}} & 0.97874 & \multicolumn{1}{c|}{848} & 882 & 0.080378 \\ \hline
\begin{tabular}[c]{@{}c@{}}GC\\ orig. vs. D$_{cos}$\end{tabular} & \begin{tabular}[c]{@{}c@{}}earn\\ vs. rest\end{tabular} & 0.1 & 1000 & \multicolumn{1}{c|}{\textbf{0.97876}} & 0.97874 & \multicolumn{1}{c|}{870} & 884 & 0.971844 \\ \hline
\begin{tabular}[c]{@{}c@{}}GC\\ orig. vs. D$_{arc}$\end{tabular} & \begin{tabular}[c]{@{}c@{}}earn\\ vs. rest\end{tabular} & 0.1 & 10 & \multicolumn{1}{c|}{\textbf{0.98263}} & 0.97851 & \multicolumn{1}{c|}{903} & \textbf{879} & 0.070112 \\ \hline
\end{tabular}%
}
\caption{Maximum values obtained for the F1 metric for each of the transformed kernels.}
\label{tab:OVR_BestF1custom}
\end{table}

Considering all the analyzed models, the best performance overall for the one-vs.-rest task is achieved by the Linear and GC kernels with an F1 score of $0.98340$, in their original versions without any transformation applied. See Table \ref{tab:OVR_BestF1overall}. In bold we marked the highest F1 value across all models. The F1 score is the same for the Linear original and GC original models.
\begin{table}[]
\centering
\begin{tabular}{|c|c|}
\hline
\textbf{Model} & \textbf{F1}      \\ \hline
Linear orig.   & \textbf{0.98340} \\ \hline
Gaussian orig. & 0.97706          \\ \hline
GC orig.       & \textbf{0.98340} \\ \hline
Linear D$_{cos}$      & 0.97751          \\ \hline
Gauss D$_{cos}$       & 0.98101          \\ \hline
GC D$_{cos}$          & 0.97874          \\ \hline
GC D$_{arc}$          & 0.97851          \\ \hline
\end{tabular}
\caption{Maximum value for F1 metric of each kernel for the one-vs.-rest task.}
\label{tab:OVR_BestF1overall}
\end{table}

An important result that we want to emphasize is the scenarios where we obtained the highest increase in performance for each kernel. These results confirm previous results obtained by Amari and Wu in \cites{AmariWu1999,WuAmari} that show the efficiency of the conformal transformation technique in improving a bad kernel. These results can be seen in table \ref{tab:OVR_HighestAccuracyIncrease}. These results were obtained using the $L^1$ norm.
\begin{table}[]
\centering
\resizebox{\textwidth}{!}{%
\begin{tabular}{|c|c|c|c|c|cc|cc|c|}
\hline
\multirow{2}{*}{\textbf{Model}} & \multirow{2}{*}{\textbf{Task}} & \multirow{2}{*}{\textbf{TFIDF}} & \multirow{2}{*}{\textbf{$\gamma$}} & \multirow{2}{*}{\textbf{C}} & \multicolumn{2}{c|}{\textbf{F1}} & \multicolumn{2}{c|}{\textbf{\#S.V.}} & \multirow{2}{*}{\textbf{\begin{tabular}[c]{@{}c@{}}P-value \\ for F1\end{tabular}}} \\ \cline{6-9}
 &  &  &  &  & \multicolumn{1}{c|}{\textbf{original}} & \textbf{custom} & \multicolumn{1}{c|}{\textbf{original}} & \textbf{custom} &  \\ \hline
\begin{tabular}[c]{@{}c@{}}Linear\\ orig. vs. D$_{cos}$\end{tabular} & \begin{tabular}[c]{@{}c@{}}grain\\ vs. rest\end{tabular} & TRUE & - & 1 & \multicolumn{1}{c|}{0.58656} & \textbf{0.96701} & \multicolumn{1}{c|}{1185} & \textbf{638} & \multicolumn{1}{l|}{\textbf{8.05E-10}} \\ \hline
\begin{tabular}[c]{@{}c@{}}Gaussian\\ orig. vs. D$_{cos}$\end{tabular} & \begin{tabular}[c]{@{}c@{}}grain\\ vs. rest\end{tabular} & FALSE & 0.0001 & 10 & \multicolumn{1}{c|}{0.00000} & \textbf{0.98101} & \multicolumn{1}{c|}{1062} & \textbf{388} & \textbf{1.98E-32} \\ \hline
\begin{tabular}[c]{@{}c@{}}GC\\ orig. vs. D$_{cos}$\end{tabular} & \begin{tabular}[c]{@{}c@{}}earn\\ vs. rest\end{tabular} & FALSE & 0.0001 & 1 & \multicolumn{1}{c|}{0.00000} & \textbf{0.97820} & \multicolumn{1}{c|}{7410} & \textbf{831} & \textbf{3.79E-37} \\ \hline
\begin{tabular}[c]{@{}c@{}}GC\\ orig. vs. D$_{arc}$\end{tabular} & \begin{tabular}[c]{@{}c@{}}earn\\ vs. rest\end{tabular} & FALSE & 0.0001 & 1 & \multicolumn{1}{c|}{0.00000} & \textbf{0.97721} & \multicolumn{1}{c|}{7410} & \textbf{816} & \textbf{3.19E-37} \\ \hline
\end{tabular}%
}
\caption{Highest increase in performance obtained by the conformal transformation for each kernel.}
\label{tab:OVR_HighestAccuracyIncrease}
\end{table}

The outcome of our study suggests that the GC kernel is on par with the Linear kernel in terms of performance for this specific task. The Linear kernel is well known to achieve good performance for high-dimensional text classification problems because most of them are linearly separable \cites{joachims1998text, hasan2022classificationF1Score}.  While the conformal transformation method can enhance a kernel’s accuracy, it does not outperform well-known kernels used in machine learning. In the cases where an increase in performance was observed, there was an imbalance in the number of positive and negative examples, with the number of positive examples being fewer. For example, considering the "grain" topic as the positive label, there are about 13 times more negative samples. For this topic, the Linear and Gaussian showed poor accuracy in some cases and could be significantly improved by the conformal transformation. This shows that the technique might prove effective for imbalanced datasets with a small number of positive samples for this particular task. 
\subsection{One vs. One Task}
For this task, we compared the performance of the SVM classification algorithm using the Linear, Gaussian and GC kernels and their transformed counterparts. 
We notice again that the procedure is especially successful in the cases where the original kernel does not perform well. 
For the Linear kernel, considering the D$_{cos}$ transformation given by \eqref{ConfTransCosine}, we see an increase in F1 metric in $42\%$ of the cases. If we also consider the cases where we have a non-significant decrease in F1 ($p-$value $\geq 0.05$) and a decrease in the number of support vectors, we see an improvement in $60\%$ of the cases. For the Gaussian kernel with D$_{cos}$ transformation given by \eqref{ConfTransCosine} applied we notice an increase of the F1 score in $70\%$ of the cases. If we also consider the cases where we have a non-significant decrease in F1 and a decrease in the number of support vectors, we see an improvement in $84\%$ of the cases.

For the GC kernel using transformations D$_{cos}$ and D$_{arc}$ given by Equations \eqref{ConfTransCosine} and \eqref{ConfTransArccos} we noticed an increase of the F1 score in $64\%$ of the cases, for both transformations. If we also consider the cases where we have a non-significant decrease in F1 and a decrease in the number of support vectors, we see an improvement in $81\%$ and $80\%$ of the cases for D$_{cos}$ and D$_{arc}$ respectively. A summary of these results can be seen in Table \ref{tab:OVO_SummaryPercIncrease}. In the Accuracy Increase column we present the percentage of cases where an increase of accuracy using F1 metric was observed. In the Efficiency Increase column besides an increase in F1 metric we additionally consider the cases where a non-significant decrease in F1 was observed accompanied by a decrease in the number of support vectors.
\begin{table}[]
\centering
\begin{tabular}{|c|c|c|}
\hline
\textbf{Model} & \textbf{\begin{tabular}[c]{@{}c@{}}Accuracy\\ Increase\\ (\% of tested\\ scenarios)\end{tabular}} & \textbf{\begin{tabular}[c]{@{}c@{}}Efficiency\\ Increase\\ (\% of tested\\ scenarios)\end{tabular}} \\ \hline
\begin{tabular}[c]{@{}c@{}}Linear\\ with D$_{cos}$\end{tabular} & 42\% & 60\% \\ \hline
\begin{tabular}[c]{@{}c@{}}Gauss\\ with D$_{cos}$\end{tabular} & \textbf{70\%} & \textbf{84\%} \\ \hline
\begin{tabular}[c]{@{}c@{}}GC\\ with D$_{cos}$\end{tabular} & 64\% & 81\% \\ \hline
\begin{tabular}[c]{@{}c@{}}GC\\ with D$_{arc}$\end{tabular} & 64\% & 80\% \\ \hline
\end{tabular}
\caption{Percentage of tested scenarios where an improvement was observed in the conformally transformed kernel compared to the original kernel.}
\label{tab:OVO_SummaryPercIncrease}
\end{table}

\textbf{Discussion and Implications}

For the one-vs.-one task, we see that the conformal transformation technique proves effective for kernels with a bad performance. It has the biggest impact on the Gaussian kernel,  improving accuracy in $70\%$ of the tested scenarios and model efficiency in $84\%$ of the tested scenarios as seen in Table \ref{tab:OVO_SummaryPercIncrease}.

If we consider the best F1 accuracy over all original models we see that these could not be further improved by the conformal transformation technique for any of the kernels. Again, this contradicts the original experiments performed in \cites{AmariWu1999, WuAmari}. Similar to the one-vs.-rest task, this could be explained by a difference in the setup of the original experiment: a more balanced dataset and fewer training samples compared to the test samples. 

The current experiment achieved the best performance for the "acq vs money-fx" task for all analyzed original kernels. For this combination of topics, the number of positive samples is about three times higher than the number of negative samples. This could indicate that the technique is not appropriate for this particular task for datasets where the positive samples outnumber the negative samples. We notice that for the Linear kernel, the performance obtained by the conformal transformation matches the one of the original kernel but with a higher number of support vectors. For the Gaussian and GC kernels we notice a negligible degradation in accuracy, not statistically significant and a decrease in the number of support vectors which can be seen as an improvement in model efficiency. An overview is presented in Table \ref{tab:OVO_BestF1Orig}. Results from this table were obtained using $L^1$ norm and  $C=1000$.
\begin{table}[]
\centering
\resizebox{\textwidth}{!}{%
\begin{tabular}{|c|c|c|c|cc|cc|l|}
\hline
\multirow{2}{*}{\textbf{Model}} & \multirow{2}{*}{\textbf{Task}} & \multirow{2}{*}{\textbf{TFIDF}} & \multirow{2}{*}{\textbf{$\gamma$}} & \multicolumn{2}{c|}{\textbf{F1}} & \multicolumn{2}{c|}{\textbf{\#S.V.}} & \multicolumn{1}{c|}{\multirow{2}{*}{\textbf{\begin{tabular}[c]{@{}c@{}}P-value\\ for F1\end{tabular}}}} \\ \cline{5-8}
 &  &  &  & \multicolumn{1}{c|}{\textbf{original}} & \textbf{custom} & \multicolumn{1}{c|}{\textbf{original}} & \textbf{custom} & \multicolumn{1}{c|}{} \\ \hline
\begin{tabular}[c]{@{}c@{}}Linear\\ orig. vs. D$_{cos}$\end{tabular} & \begin{tabular}[c]{@{}c@{}}acq\\ vs. money-fx\end{tabular} & TRUE & 1 & \multicolumn{1}{c|}{\textbf{0.99874}} & \textbf{0.99874} & \multicolumn{1}{c|}{477} & 488 & \multicolumn{1}{c|}{-} \\ \hline
\begin{tabular}[c]{@{}c@{}}Gaussian \\ orig. vs. D$_{cos}$\end{tabular} & \begin{tabular}[c]{@{}c@{}}acq\\ vs. money-fx\end{tabular} & FALSE & 0.01 & \multicolumn{1}{c|}{\textbf{0.99916}} & 0.99852 & \multicolumn{1}{c|}{360} & \textbf{278} & 0.082822 \\ \hline
\begin{tabular}[c]{@{}c@{}}GC\\ orig. vs. D$_{cos}$\end{tabular} & \begin{tabular}[c]{@{}c@{}}acq\\ vs. money-fx\end{tabular} & FALSE & 0.001 & \multicolumn{1}{c|}{\textbf{0.99832}} & 0.99831 & \multicolumn{1}{c|}{321} & \textbf{299} & 0.990841 \\ \hline
\begin{tabular}[c]{@{}c@{}}GC\\ orig. vs. D$_{arc}$\end{tabular} & \begin{tabular}[c]{@{}c@{}}acq\\ vs. money-fx\end{tabular} & FALSE & 0.001 & \multicolumn{1}{c|}{\textbf{0.99832}} & 0.99831 & \multicolumn{1}{c|}{321} & \textbf{306} & 0.995401 \\ \hline
\end{tabular}%
}
\caption{Maximum values of F1 metric for each original kernel.}
\label{tab:OVO_BestF1Orig}
\end{table}

Analyzing the best F1 score of all transformed kernels, we notice in $6$ out of $8$ cases an improvement in F1 metric compared to the original kernel. Out of these $6$, $3$ are statistically significant. In the remaining $2$ out of the $8$ scenarios, we notice a non-significant decrease in F1 but a smaller number of support vectors, which could still indicate an improvement in model efficiency. These results are summarized in Table \ref{tab:OVO_BestF1Custom}. Results from this table were obtained using the $L^1$ norm.
\begin{table}[]
\centering
\resizebox{\textwidth}{!}{%
\begin{tabular}{|c|c|c|c|c|cc|cc|c|}
\hline
\multirow{2}{*}{\textbf{Model}} & \multirow{2}{*}{\textbf{Task}} & \multirow{2}{*}{\textbf{TFIDF}} & \multirow{2}{*}{\textbf{Gamma}} & \multirow{2}{*}{\textbf{C}} & \multicolumn{2}{c|}{\textbf{F1}} & \multicolumn{2}{c|}{\textbf{\#S.V.}} & \multirow{2}{*}{\textbf{\begin{tabular}[c]{@{}c@{}}P-value\\ for F1\end{tabular}}} \\ \cline{6-9}
 &  &  &  &  & \multicolumn{1}{c|}{\textbf{original}} & \textbf{custom} & \multicolumn{1}{c|}{\textbf{original}} & \textbf{custom} &  \\ \hline
\begin{tabular}[c]{@{}c@{}}Linear\\ orig. vs. D$_{cos}$\end{tabular} & \begin{tabular}[c]{@{}c@{}}acq vs.\\ money-fx\end{tabular} & TRUE & 1 & 1 & \multicolumn{1}{c|}{0.93697} & \textbf{0.99874} & \multicolumn{1}{c|}{1196} & \textbf{482} & \textbf{4.08E-13} \\ \hline
\begin{tabular}[c]{@{}c@{}}Linear\\ orig. vs. D$_{cos}$\end{tabular} & \begin{tabular}[c]{@{}c@{}}acq vs.\\ money-fx\end{tabular} & TRUE & 1 & 10 & \multicolumn{1}{c|}{0.99726} & \textbf{0.99874} & \multicolumn{1}{c|}{667} & \textbf{484} & 0.110155 \\ \hline
\begin{tabular}[c]{@{}c@{}}Linear\\ orig. vs. D$_{cos}$\end{tabular} & \begin{tabular}[c]{@{}c@{}}acq vs.\\ money-fx\end{tabular} & TRUE & 1 & 100 & \multicolumn{1}{c|}{0.99852} & \textbf{0.99874} & \multicolumn{1}{c|}{479} & 487 & 0.329877 \\ \hline
\begin{tabular}[c]{@{}c@{}}Gaussian\\ orig. vs. D$_{cos}$\end{tabular} & \begin{tabular}[c]{@{}c@{}}acq vs.\\ money-fx\end{tabular} & FALSE & 0.1 & 1000 & \multicolumn{1}{c|}{0.99873} & \textbf{0.99895} & \multicolumn{1}{c|}{276} & 284 & 0.579204 \\ \hline
\begin{tabular}[c]{@{}c@{}}GC\\ orig. vs. D$_{cos}$\end{tabular} & \begin{tabular}[c]{@{}c@{}}acq vs.\\ money-fx\end{tabular} & FALSE & 0.001 & 1000 & \multicolumn{1}{c|}{\textbf{0.99832}} & 0.99831 & \multicolumn{1}{c|}{321} & \textbf{299} & 0.990841 \\ \hline
\begin{tabular}[c]{@{}c@{}}GC\\ orig. vs. D$_{arc}$\end{tabular} & \begin{tabular}[c]{@{}c@{}}acq vs.\\ money-fx\end{tabular} & FALSE & 0.001 & 1000 & \multicolumn{1}{c|}{\textbf{0.99832}} & 0.99831 & \multicolumn{1}{c|}{321} & \textbf{306} & 0.995401 \\ \hline
\begin{tabular}[c]{@{}c@{}}GC\\ orig. vs. D$_{arc}$\end{tabular} & \begin{tabular}[c]{@{}c@{}}acq vs.\\ money-fx\end{tabular} & FALSE & 0.01 & 1 & \multicolumn{1}{c|}{0.87619} & \textbf{0.99831} & \multicolumn{1}{c|}{1275} & \textbf{302} & \textbf{8.80E-26} \\ \hline
\begin{tabular}[c]{@{}c@{}}GC\\ orig. vs. D$_{arc}$\end{tabular} & \begin{tabular}[c]{@{}c@{}}acq vs.\\ money-fx\end{tabular} & FALSE & 0.01 & 10 & \multicolumn{1}{c|}{0.99517} & \textbf{0.99831} & \multicolumn{1}{c|}{710} & \textbf{301} & \textbf{0.001578} \\ \hline
\end{tabular}%
}
\caption{Maximum values of F1 metric for each transformed kernel.}
\label{tab:OVO_BestF1Custom}
\end{table}
The best overall F1 score for the one-vs.-one task was achieved by the Gaussian original kernel, at $0.99916$. This could not be surpassed by the conformal transformation, as we can see in Table  \ref{tab:OVO_BestF1Overall}. In bold we marked the highest F1 value across all models.
\begin{table}[]
\centering
\begin{tabular}{|c|c|}
\hline
\textbf{Model} & \textbf{F1} \\ \hline
Linear orig. & 0.99874 \\ \hline
Gaussian orig. & \textbf{0.99916} \\ \hline
GC orig. & 0.99832 \\ \hline
Linear D$_{cos}$ & 0.99874 \\ \hline
Gauss D$_{cos}$ & 0.99895 \\ \hline
GC D$_{cos}$ & 0.99831 \\ \hline
GC D$_{arc}$ & 0.99831 \\ \hline
\end{tabular}
\caption{Maximum value of F1 metric for each kernel for the one-vs.-one task.}
\label{tab:OVO_BestF1Overall}
\end{table}

An important result that we want to emphasize is the scenarios where we obtained the highest increase in performance for each kernel. These results confirm previous results obtained by Amari and Wu in \cites{AmariWu1999,WuAmari} that show the efficiency of the conformal transformation technique in improving a bad kernel. These results can be seen in Table \ref{tab:OVO_HighestAccuracyIncrease}.
\begin{table}[]
\centering
\resizebox{\textwidth}{!}{%
\begin{tabular}{|c|c|c|c|c|c|cc|cc|c|}
\hline
\multirow{2}{*}{\textbf{Model}} & \multirow{2}{*}{\textbf{Task}} & \multirow{2}{*}{\textbf{Norm}} & \multirow{2}{*}{\textbf{TFIDF}} & \multirow{2}{*}{\textbf{$\gamma$}} & \multirow{2}{*}{\textbf{C}} & \multicolumn{2}{c|}{\textbf{F1}} & \multicolumn{2}{c|}{\textbf{\#S.V.}} & \multirow{2}{*}{\textbf{\begin{tabular}[c]{@{}c@{}}P-value\\ for F1\end{tabular}}} \\ \cline{7-10}
 &  &  &  &  &  & \multicolumn{1}{c|}{\textbf{original}} & \textbf{custom} & \multicolumn{1}{c|}{\textbf{original}} & \textbf{custom} &  \\ \hline
\begin{tabular}[c]{@{}c@{}}Linear\\ orig. vs. D$_{cos}$\end{tabular} & \begin{tabular}[c]{@{}c@{}}grain\\ vs. earn\end{tabular} & $L^2$ & TRUE & - & 1 & \multicolumn{1}{c|}{0.67719} & \textbf{0.99737} & \multicolumn{1}{c|}{1024} & \textbf{416} & \textbf{1.10E-10} \\ \hline
\begin{tabular}[c]{@{}c@{}}Gaussian\\ orig. vs. D$_{cos}$\end{tabular} & \begin{tabular}[c]{@{}c@{}}grain\\ vs. earn\end{tabular} & $L^1$ & FALSE & 0.0001 & 100 & \multicolumn{1}{c|}{0.00000} & \textbf{0.99743} & \multicolumn{1}{c|}{1072} & \textbf{232} & \textbf{1.70E-43} \\ \hline
\begin{tabular}[c]{@{}c@{}}GC\\ orig. vs. D$_{cos}$\end{tabular} & \begin{tabular}[c]{@{}c@{}}crude\\ vs. money-fx\end{tabular} & $L^1$ & FALSE & 0.0001 & 1 & \multicolumn{1}{c|}{0.00000} & \textbf{0.99646} & \multicolumn{1}{c|}{1064} & \textbf{205} & \textbf{2.80E-42} \\ \hline
\begin{tabular}[c]{@{}c@{}}GC\\ orig. vs. D$_{arc}$\end{tabular} & \begin{tabular}[c]{@{}c@{}}money-fx\\ vs. earn\end{tabular} & $L^1$ & FALSE & 0.0001 & 1 & \multicolumn{1}{c|}{0.00000} & \textbf{0.99611} & \multicolumn{1}{c|}{1250} & \textbf{248} & \textbf{8.42E-39} \\ \hline
\end{tabular}%
}
\caption{Highest increase in F1 metric obtained by the conformal transformation of each kernel.}
\label{tab:OVO_HighestAccuracyIncrease}
\end{table}

The outcome of our study suggests that the Gaussian kernel without the conformal transformation achieved the best performance for this specific task. While the conformal transformation method can enhance a kernel’s accuracy, it does not outperform the Gaussian kernel which is a well known for its flexibility and ability to handle non-linear data. For the topic combinations where the highest increase of performance was observed, "grain vs. earn", "crude vs. money-fx" and "money-fx vs. earn" we see that the number of positive samples is smaller than the number of negative samples (up to five times), indicating that this technique can be useful for this particular task when dealing with this type of imbalanced data. 

\section{Conclusions and Future Work}\label{SectionConclusions}
In this article we have studied the effect of conformal transformation of kernels for two types of binary classification tasks, one-vs.-rest and one-vs.-one. We analyzed the performance of this technique for some well known kernels, the Linear and Gaussian kernels using the cosine distance in the conformal transformation which is appropriate for text documents. 
By considering text documents as points on a multinomial statistical manifold that can be isometrically mapped on the sphere, we proposed a new GC kernel and a conformal transformation using the geodesic distance on the sphere. The proposed method is inspired by information geometry. 
These experiments show that this type of transformation can increase the performance of a given kernel, for both types of binary classification tasks especially if the original kernel performs badly. The procedure proves effective when applied to the Linear, Gaussian and GC kernels, increasing or matching the performance of the original kernel in $60\%$ of cases for the Linear kernel, $84\%$ for the Gaussian kernel and $78\%$ for the GC kernel. However, the best performance achieved by the transformed kernels could not surpass the best performance achieved by the original kernels.  

Since there is no general solution for all classification tasks and the best solution depends on the problem to solve and the dataset properties, we want to note that the proposed method has both advantages and  drawbacks. The primary benefit lies in its ability to shorten the time needed for hyperparameter tuning and the search for the optimal kernel, bypassing time-consuming techniques such as cross-validation. With this method, one can begin experimenting with any kernel, even those that may initially perform poorly, as the technique has the capacity to "correct" them. In fact we have showed the method works especially well on bad kernels. There is no need for extra tuning of parameters compared to the original kernel. Another advantage is that it performs well on imbalanced datasets where negative samples outnumber positive ones. As disadvantages, we observe the computational time. The method requires a first pass using any kernel, followed by a second step with the transformed kernel, which will increase the computational time. Another disadvantage is the complexity. It is a custom kernel, so it needs to be implemented as opposed to using an already available "off-the-shelf" solution. From the experiments on both synthetic and real datasets, the technique seems more effective on lower-dimensional non-linear data than on high-dimensional data like text documents, for which mapping into an even higher dimensional space does not necessarily yield substantial performance gains. As a consequence, it could not exceed the performance of already established benchmarks set by traditional kernels, like the Linear or Gaussian kernels in the context of text documents classification.

A potential future direction would be to consider other representation schemes for the documents that capture the semantic meaning of words or phrases. An overview of such techniques can be found in \cites{pal2018semantic}. We can explore more the geometry of the multinomial statistical manifold and consider the use of the Hellinger distance in the conformal transformation, which is related to the geodesic distance \eqref{arccosdist} by Equation \ref{eq:arc}.
\begin{equation} \label{eq:arc}
 d_H(\bm \theta, \bm \theta') = 2 \sin{d(\bm \theta, \bm \theta')/4}
\end{equation}
and when $\bm \theta $ and $\bm \theta'$ are close, $d_H(\bm \theta, \bm \theta') \approx \frac{1}{2}{d(\bm \theta, \bm \theta')}$ \cites{diffusion, calin}. Another possibility is to consider the averaged Kullback-Leibler divergence that gives a symmetric measure that can be used as a distance \cite{huang2008similarity}. The use of the Diffusion kernel on the hypersphere could be further explored in the approximated form given in \cite{diffusion} or \cite{zhang2005text}, or using the exact formula proposed in \cite{zhao2018exact}. 

The classification results obtained in the current study on the Reuters dataset partially confirm the findings reported by Amari and Wu \cites{AmariWu1999, WuAmari}. While these earlier studies, conducted on synthetic and lower-dimensional data, demonstrated that conformal transformations can significantly enhance the performance of a kernel, our experiments indicate that the efficacy of such transformations is somewhat dependent on the nature of the dataset. Specifically, when applied to the high-dimensional and sparse data characteristic of the Reuters corpus, the improvements provided by conformal transformations were less pronounced than those observed in lower-dimensional settings. This suggests that while conformal transformations remain a powerful tool for kernel enhancement, their impact may be moderated by factors intrinsic to the data, such as dimensionality and sparsity.

\bibliographystyle{amsxport}
\clearpage
\bibliography{References}  
\end{document}